\documentclass[11pt]{article}
\usepackage[margin=1in]{geometry}
\usepackage{amsmath}
\usepackage{amsfonts}
\usepackage{amssymb}
\usepackage{tikz-cd}
\usepackage{enumitem}
\usepackage{booktabs}
\usepackage{tabularx}
\usepackage{array}
\usepackage{longtable}
\usepackage{amsthm}
\usepackage[square,comma,authoryear]{natbib}

\newtheorem{theorem}{Theorem}
\newtheorem{proposition}{Proposition}
\newtheorem{definition}{Definition}
\newtheorem{remark}{Remark}

\newtheorem{corollary}{Corollary}

\title{Towards Error-Centric Intelligence II:\\Energy-Structured Causal Models}
\author{Marcus A. Thomas \thanks{This work was conducted independently and does not represent the views of Memorial Sloan Kettering.} \\\texttt{thomm15@mskcc.org}}
\date{\today}

\begin{document}

\maketitle

\begin{abstract}
Contemporary machine learning optimizes for predictive accuracy, yet systems that achieve state-of-the-art performance remain causally opaque: their internal representations provide no principled handle for intervention. We can retrain such models, but we cannot surgically edit specific mechanisms while holding others fixed, because learned latent variables lack causal semantics. We argue for a conceptual reorientation: intelligence is the ability to build and refine explanations—falsifiable claims about manipulable structure that specify what changes and what remains invariant under intervention. Explanations subsume prediction but demand more: causal commitments that can be independently tested and corrected at the level of mechanisms. 

We introduce \emph{computational explanations}: mappings from observations to intervention-ready causal accounts. We instantiate these explanations with \emph{Energy\textendash Structured Causal Models} (E\textendash SCMs), in which mechanisms are expressed as constraints (energy functions or vector fields) rather than explicit input–output maps, and interventions act by local surgery on those constraints. This shift makes internal structure manipulable at the level where explanations live: which relations must hold, which can change, and what follows when they do. We provide concrete instantiations of the structural-causal principles LAP and ICM in the E-SCM context, and also argue that empirical risk minimization systematically produces fractured, entangled representations—a failure we analyze as gauge ambiguity in encoder–energy pairs. 

Finally, we show that under mild conditions, E\textendash SCMs recover standard SCM semantics while adding declarative interventional structure suited to learned representations. Building on Part~I's principles (LAP, ICM, CAP) and its definition of intelligence as explanation-building under criticism, this paper offers a formal language for causal reasoning in systems that aspire to understand, not merely to predict. Empirical demonstrations are deferred to future work, as the primary goal here is to establish a well-posed mathematical foundation.
\end{abstract}

\section{Introduction}
This paper is Part~II in a two-part series. Part~I \citep{thomas2025towards} developed a foundation for error-centric intelligence. It defined intelligence as the ability to create and refine explanations\textendash interdependent hypotheses that make hard-to-vary claims (either implicit or explicit) about the world, including interventions and counterfactuals, and that invite criticism through tests, e.g., observation. Part~I also articulated three structural principles: the Locality--Autonomy Principle (LAP), geometric Independent Causal Mechanisms (ICM), and the Compositional Autonomy Principle (CAP) that collectively address a number of technical challenges (e.g., fractured, entangled representations) and facilitate error correction at the level of concepts.

Part~II does not directly tackle artificial general intelligence (AGI) as defined in Part~I. Our goal here is to describe one possible formalism for modeling explanations within learning systems. The approach we take, if it proves useful, may form the foundation of a future AGI, though there are likely very many alternative paths. Fortunately, any developments in explanatory modeling essential to an AGI may also be useful for narrow AI, e.g., to the challenges of generalization, continual learning and catastrophic forgetting. Today’s narrow AI systems are characterized by high competence but low (or no) intelligence and so the payoff may be substantial: explanations enhance statistical prediction with editable, testable structure and provide the substrate for principled self-correction.

We model explanations within the Causal Mechanics program via \emph{Energy\textendash Structured Causal Models} (E\textendash SCMs). Mechanisms are represented not as explicit input–output maps but as energy based constraints, functions whose minima (or fixed points of the induced flow) define admissible or preferred latent configurations. Interventions act by editing these mechanisms, providing interventional semantics for latent causal variables.

In a traditional SCM, each endogenous variable is produced by a function $X_i=f_i(X_{\mathrm{PA}(i)},U_i)$; editing the world requires reprogramming $f_i$ (and often its neighbors) to maintain consistency. In an E\textendash SCM, the same semantics are expressed declaratively by local energies $E_i$ (and optional global terms) whose equilibria define admissible configurations. Interventions become edits of constraints: hard actions clamp variables by imposing infinite barriers; soft actions deform local energies; disjunctive actions restrict values to sets. Abduction, intervention, and prediction are executed by solving for equilibria before and after local surgery, optionally without invoking probabilities. This shift makes the internal `story' manipulable \textit{at the level at which explanations live}: which relations must hold, which can change, and what follows when they do.

Explanatory leverage arises from two further ingredients. First, the energy formalism makes internal commitments manipulable: after local surgery, equilibria must re-establish global consistency, so disagreements that were implicit under observation become visible as failures to re-equilibrate along the intended causal pathways. Second, LAP can be enforced directly in energy space: cross-partial diagnostics and penalties suppress illicit dependence of a mechanism's effective energy on non-descendants, turning modularity into something that can be checked and corrected. We defer any instrumentation of the latent landscape to later sections; here the point is conceptual—constraints, not computations, are the unit of edit, and modularity is a condition on those constraints.

The approach is conservative with respect to causal semantics. Under mild locality and well-posedness assumptions, the abduction--intervention--prediction behavior of an E\textendash SCM coincides with that of an induced SCM defined by the blockwise argmin maps. Thus the proposal inherits the identification calculus and counterfactual logic of standard SCMs while adding a declarative interventional layer that is better matched to learned latent representations.

\paragraph{Causality and Invariance.}
Causality is not exhausted by the fact that systems change, but by the form of change permitted by what remains invariant. Change is the empirical signature of causality; invariance is its grammar. To call a relation causal is to claim that when one aspect of a system is altered, others will vary in a lawlike manner while the governing constraints persist. The invariants define the space of admissible transformations—the background against which interventions become meaningful. Without them, the distinction between a causal alteration and a wholesale rewriting of the world collapses. In Energy\textendash Structured Causal Models, each mechanism-term encodes a \emph{constraint}; when that constraint is held fixed across a class of interventions, the stability of the relevant predictions is the corresponding \emph{invariant}. Constraints are prescriptive (what must hold); invariants are descriptive (what in fact remains stable when non-governing aspects are edited). Interventions act by selectively relaxing or replacing terms, and the ensuing equilibrium shift reveals how change propagates through what endures. The causal content of a model therefore lies neither in transformations alone nor in static structure, but in the reciprocity between what is allowed to vary and what must remain fixed for the variation to make sense.

\paragraph{Change Without Invariance.}
Change in the absence of constraint-backed invariance carries no causal content. A model that reproduces transitions without identifying the governing constraints captures correlation, not cause: it states what happens but not what must happen under specified alterations. Such models can fit observations yet fail to generalize, because the very background that renders an intervention intelligible is missing. By contrast, invariance without change is sterile: a system frozen under all perturbations expresses no mechanism, only identity. Causal structure arises between these extremes. It is the delineation of what can change and what must not—the partitioning of variability induced by constraints—that allows explanations to extend beyond the data that first motivated them.

\paragraph{Effective theories.}
We assert that higher level explanations are legitimate when they identify manipulable handles and constraints that reliably structure change at the chosen scale, rather than by appeal to reduction to lower level primitives. An effective theory stakes a conjecture about what remains fixed and what changes under edits we can actually perform, and earns its status by withstanding criticism in the form of micro variation, contextual shifts, and competing mechanism proposals. The 'right` level is also a conjectural assertion, e.g., the coarsest description at which interventions are operationally available, enabling conditions are stable, and predictions exhibit reach across diverse realizations beneath them. 

Interactions exist across abstraction levels: downward influence is a macro edit that narrows the set of admissible micro trajectories or configurations; upward influence occurs when micro changes violate enabling conditions and void the macro claim; sideways influence arises when distinct macro mechanisms couple so that their constraints conflict or reinforce.\footnote{By macro we mean variables, mechanisms, and edits defined at a coarser descriptive scale where interventions are directly formulable (policies, protocols, rules); by micro we mean finer-scale realizations whose many configurations can implement the same macro relation.} These directions are not metaphysical categories but testable commitments about how interventions commute with refinement and aggregation. 

In biology, a dosing protocol that consistently induces predicted pathway responses across cell lines and laboratories indicates a real (but tentative) macro mechanism; in epidemiology, altering contact structure to shift transmission while biological parameters remain fixed demonstrates causal content at the population level; in social systems, editing a matching rule or market policy that predictably changes allocation outcomes under stable institutional constraints establishes a legitimate macro cause. Effective theories are therefore not shortcuts around microphysics but disciplined conjectures about where invariants live and where interventions close\footnote{The level at which interventions form a self-contained and intelligible causal system — where you can define manipulations and reliably predict their consequences without appealing to lower-level details.}, defended by their editability, reach, and resilience under severe tests.

\paragraph{Limitations of Existing Neural SCM Approaches.}
Neural parameterizations of structural mechanisms have improved the expressiveness and trainability of causal models, yet most frameworks still reason over observables \citep{Pawlowski2020DeepSC,Xia2021TheCC,Xia2022NeuralCM}. Latent variables typically act as noise surrogates rather than causal coordinates, leaving the semantics of learned representations underspecified. When interventions are defined only on observables, distinct internal mechanisms can agree on all observable interventional distributions, making implicit errors unreachable and masking structural mismatches. For systems that aspire to explanatory or self-correcting behavior, this omission leaves no means of criticism at the level where mechanisms operate.

Normalizing-flow, causal-VAE, and diffusion-based families share this limitation despite their architectural diversity. Flow models link likelihood training to per-node mechanisms through invertibility, but the resulting bijections do not guarantee that causal dependencies are modular or that edits are local \citep{Khemakhem2021CausalAF,Javaloy2023CausalNF}. Causal VAEs introduce structured decoders and encourage disentanglement, yet within-concept fracture remains possible and often compensated by downstream heads \citep{Yang2021CausalVAE}. Diffusion-based models train conditional denoisers that capture interventional distributions over observables, but without an explicit notion of latent surgery or mechanism replacement \citep{Sanchez2022DiffusionCM,Chao2023InterventionalAC}. Across these families, causal reasoning remains tied to surface variables, not to editable internal relations.

These methods also inherit representational pathologies. Likelihood- and ELBO-driven objectives permit fractured or entangled latent representations  \citep{kumar2025questioning} that satisfy observational fit while violating locality and modularity. Observational training provides weak pressure—if any—to align the learned hypothesis space with (conjectural) causal structure. Without explicit diagnostics, such violations remain hidden: global consistency is restored numerically rather than mechanistically, and the resulting models generalize poorly under intervention or context shift.

Finally, existing frameworks lack edit semantics and locality diagnostics. Intervening on one component typically requires global retraining or reparameterization, and there is no standard measure of how far edits propagate or whether causal reach is preserved. Evaluation focuses on fit and limited benchmark interventions rather than on stability under mechanism edits. The consequence is a persistent gap between statistical competence and explanatory adequacy: systems predict well within distribution but fail to reveal or correct their internal errors. Closing this gap requires causal reasoning in latent space—reasoning over editable, testable mechanisms whose behavior under intervention can be directly examined.

\paragraph{Scope and contributions.}
We add an explanatory layer with practical semantics on top of deep neural substrates, not a universal architecture. We introduce \emph{computational explanations} instantiated by Energy–Structured Causal Models (E–SCMs): mechanisms as constraints, interventions as local surgeries, equilibria restoring consistency. Contributions:

\begin{enumerate}[leftmargin=*,itemsep=0.2em,topsep=0.2em]
  \item Section~\ref{sec:ESCMs} formalizes static and dynamic E–SCMs, intervention semantics via energy edits, and well-posedness of equilibria.
  \item Section~\ref{sec:CausalOps} defines probability-optional abduction–action–prediction and hard/soft and set-valued interventions.
  \item Section~\ref{sec:LAP} operationalizes modularity with LAP diagnostics and penalties that expose non-descendant influence.
  \item Section~\ref{sec:ICM} enforces parameter-space separability with ICM penalties and commuting-flow witnesses.
  \item Section~\ref{sec:GeometricStructure} gives a Riemannian view of equilibria and analyzes fractured/entangled representations via gauge freedom.
  \item Section~\ref{sec:Explanations} introduces \emph{computational explanations} as mappings from observations to intervention-ready causal accounts.
  \item Section~\ref{sec:Integration_with_DL} sketches an implementation path: adaptors, mechanisms, actuators, and probes for integration with deep models.
\end{enumerate}

\section{E\textendash SCM Framework}
\label{sec:ESCMs}
\subsection{Static vs.\ Dynamic E\textendash SCMs}

Energy\textendash structured causal models come in two complementary forms that address different questions about a system. Static E\textendash SCMs treat mechanisms as scalar energy terms $E_i(z_i \mid z_{\mathrm{PA}(i)})$ whose minima define equilibria. They ask: what configurations are compatible with the mechanisms and constraints? In this view, a system state is an equilibrium (an energy minimum), and interventions deform the energy landscape before re-equilibration. Dynamic E\textendash SCMs describe mechanisms as vector fields $F_i(z_i,z_{\mathrm{PA}(i)})$ that generate trajectories via $\frac{dz_i}{dt}=F_i$. They ask: how does the system evolve over time under interventions or perturbations? This formulation makes transient behavior, path dependence, and rates of change central objects of analysis.

Static formulations are appropriate when equilibrium reasoning suffices, such as studying stable operating points or feasibility under constraints. Dynamic formulations are preferable when timing, transient responses, or control policies matter. Both cases are unified under the differential LAP from Part I.

\paragraph{Static E\textendash SCMs.}

Rather than factor a joint law into conditionals, we specify an additive energy:

\begin{definition}[Static E\textendash SCM]
A static E\textendash SCM is a tuple
\[
\bigl(\mathcal G,\ \mathbf Z,\ \mathbf U,\ \{E_i\}_{i=1}^n,\ \{E_{U_i}\}_{i=1}^n,\ E_{\mathrm{global}}\bigr),
\]
where $\mathcal G$ is a DAG on endogenous variables (e.g., latents) $\mathbf Z=(Z_1,\dots,Z_n)$, exogenous variables $\mathbf U=(U_1,\dots,U_n)$, parent set $\mathrm{PA}(i)$, local energies
\[
E_i:\ \mathcal Z_i\times \mathcal Z_{\mathrm{PA}(i)}\times \mathcal U_i\to\mathbb R,
\qquad
\mathcal Z_{\mathrm{PA}(i)}:=\prod_{j\in \mathrm{PA}(i)}\mathcal Z_j,
\]
exogenous energies $E_{U_i}:\mathcal U_i\to\mathbb R$, and an optional global term
\[
E_{\mathrm{global}}:\ \prod_{i=1}^n\mathcal Z_i \times \prod_{i=1}^n \mathcal U_i \to \mathbb R.
\]
The total energy at $(\mathbf z,\mathbf u)$ is
\[
E(\mathbf z,\mathbf u)=\sum_{i=1}^n E_i\bigl(z_i\mid z_{\mathrm{PA}(i)},u_i\bigr)\;+\;\sum_{i=1}^n E_{U_i}(u_i)\;+\;E_{\mathrm{global}}(\mathbf z,\mathbf u).
\]
\end{definition}

\noindent\textit{Consequences.}
(i) Locality: $E_i$ depends only on $z_i$, its parents, and $u_i$.
(ii) Exogenous structure is encoded by $\{E_{U_i}\}$ without invoking probabilities.
(iii) System-wide couplings live in $E_{\mathrm{global}}$.
(iv) Equilibria are the stationary points $\nabla_{\mathbf z,\mathbf u}E=0$ (minima in the well-posed case).
(v) The additive decomposition plays the role of multiplicative factorization in probabilistic SCMs.

\paragraph{Dynamic E\textendash SCMs.}
Dynamic E\textendash SCMs specify vector-field mechanisms $F_i$ and evolve
\[
\frac{dz_i}{dt}=F_i\bigl(z_i, z_{\mathrm{PA}(i)}, u_i\bigr),\qquad i=1,\dots,n.
\]
Interventions rewrite these laws: hard actions replace $F_i$ by a feedback control enforcing $z_i(t)\equiv z_i^\ast$; soft actions deform $F_i$ continuously. The differential LAP carries over via Lie derivatives of the relevant fields and constraints, unifying the static (steady-state) and dynamic (transient) views.

\paragraph{Reduction Theorem.}
\begin{theorem}[Reduction to SCM semantics (informal)]
Under standard locality and well-posedness (Assumptions A1--A4; see Appendix~\ref{sec:appendix-reduction-esm}), equilibria of a (possibly edited) E\textendash SCM coincide with solutions of an induced SCM with the corresponding surgical edits. Hence observational, interventional, and counterfactual answers agree with the induced SCM.
\end{theorem}

\noindent\emph{Proof sketch.} Each local energy induces a blockwise argmin map; the unique equilibrium is the unique fixed point of these maps. Local surgeries replace only the edited block(s); details in Appendix~\ref{sec:appendix-reduction-esm}.

\subsection{Advantages of the energy-structured approach}
The contrast at stake is not probability versus determinism but what serves as the unit of edit. In a purely distributional view, models compose by multiplying local factors and renormalizing. This is excellent for building and for inference, yet it does not by itself yield composable mechanism edits with causal locality. An edit should be stated where a mechanism lives, should propagate only along descendants absent declared global couplings, and should have a meaning that does not depend on the current coordinatization of the latent space. Energy–structured models meet these requirements by declaring mechanisms as constraints and using equilibrium to re-establish consistency after local surgery.

The probability picture makes \(\mathrm{do}(\cdot)\) difficult primarily as a matter of computation. Consider a factorized density
\[
p(x)\;=\;\frac{1}{Z(\theta)}\prod_{i}\psi_i\!\big(x_{S_i};\theta_i\big),
\qquad
Z(\theta)\;=\;\int \prod_{i}\psi_i\!\big(x_{S_i};\theta_i\big)\,dx.
\]
If a single factor is edited, \(\psi_j\mapsto \tilde\psi_j\), the new law is
\[
p'(x)\;=\;\frac{1}{Z'(\theta')}\,\tilde\psi_j\!\big(x_{S_j};\theta'_j\big)\!\!\prod_{i\neq j}\psi_i\!\big(x_{S_i};\theta_i\big),
\qquad
Z'=\int \tilde\psi_j \prod_{i\neq j}\psi_i\,dx.
\]
For any coordinate set \(T\) disjoint from \(S_j\),
\[
p'(x_T)-p(x_T)\;=\;\int \Big(\frac{1}{Z'}\tilde\psi_j-\frac{1}{Z}\psi_j\Big)\prod_{i\neq j}\psi_i\;dx_{\setminus T},
\]
so a nominally local change propagates through the normalizer and shifts non-descendant marginals. Conditioning is computationally simple but semantically different,
\[
p(x\mid Z_j=z^\star)
\;=\;
\frac{p(x)\,\delta(z_j - z^\star)}
     {\int p(x)\,\delta(z_j - z^\star)\,dx},
\]
because it restricts an unchanged model and leaves the incoming influence into \(Z_j\) intact. The correct probabilistic semantics for a hard action is to rewrite the mechanism,
\[
p^{\mathrm{do}}(x)\;\propto\;\delta\!\big(z_j-z^\star\big)\,\prod_{i\neq j}\psi_i\!\big(x_{S_i};\theta_i\big),
\]
but computing with this object is awkward in high dimension: the Dirac (or an indicator for a feasible set) defines a measure on a lower-dimensional manifold; practical surrogates replace \(\delta\) by a sharp kernel, inflate curvature, and destabilize optimization; and the changed normalizer again entangles non-descendants.

The energy picture makes \(\mathrm{do}(\cdot)\) directly computable. A hard action is a barrier on the mechanism term \(E_j\) that forbids values other than \(z^\star\); a soft action is a deformation of \(E_j\). One then re-equilibrates. Equilibrium, not renormalization, restores consistency, and under LAP only descendants move. If probability is desired, a law on exogenous variables is pushed through the edited equilibrium map; observational, interventional, and counterfactual distributions are recovered without surrendering causal locality of edits to a global normalizer.

\section{Causal Operations}
\label{sec:CausalOps}
\subsection{Interventions}

Interventions follow standard graph surgery: $do(Z_j = z_j^*)$ removes incoming arrows to $Z_j$ and fixes its value. In energy formalism, this replaces local energy $E_j$ with infinite barriers outside the clamped value.

\begin{definition}[Hard Intervention in Static E\textendash SCM]
Let
\[
E(\mathbf z, \mathbf u) = \sum_{i=1}^{n} E_i\bigl(z_i \mid z_{\mathrm{PA}(i)}, u_i\bigr) + \sum_{i=1}^{n} E_{U_i}(u_i) + E_{\mathrm{global}}(\mathbf z, \mathbf u)
\]
be static E\textendash SCM energy. Hard intervention $do(Z_j = z_j^*)$ yields modified energy:
\[
E^{do(Z_j = z_j^*)}(\mathbf z, \mathbf u) =
\begin{cases}
  \sum\limits_{i \neq j} E_i\bigl(z_i \mid z_{\mathrm{PA}(i)}, u_i\bigr) + \sum_{i=1}^{n} E_{U_i}(u_i) + E_{\mathrm{global}}(\mathbf z, \mathbf u), & \text{if } z_j = z_j^*, \\
  +\infty, & \text{otherwise.}
\end{cases}
\]
\end{definition}

Since $Z_j$ remains a parent of its children, downstream potentials still depend on $z_j$, but the value is now exogenous. Energy perturbations propagate forward through the DAG while non-descendants remain unaffected (provided cross-partials are suppressed).

Soft interventions modify mechanisms without deleting edges or clamping values:

\begin{definition}[Soft Intervention in Static E\textendash SCM]
Let $E_j^{\text{original}}$ and $E_j^{\text{intervene}}$ denote original and modified local energies. Soft intervention is:
\[
\widetilde{E}_j(z_j \mid z_{\mathrm{PA}_j}, u_j) = (1 - \lambda) E_j^{\text{original}}(z_j \mid z_{\mathrm{PA}_j}, u_j) + \lambda E_j^{\text{intervene}}(z_j \mid z_{\mathrm{PA}_j}, u_j), \quad \lambda \in [0,1]
\]
with total energy:
\[
E^{\text{soft}}(\mathbf z, \mathbf u) = \sum_{i \neq j} E_i\bigl(z_i \mid z_{\mathrm{PA}(i)}, u_i\bigr) + \widetilde{E}_j(z_j \mid z_{\mathrm{PA}_j}, u_j) + \sum_{i=1}^{n} E_{U_i}(u_i) + E_{\mathrm{global}}(\mathbf z, \mathbf u).
\]
\end{definition}

This encompasses biasing influences and mechanism shifts, recovering hard interventions as $\lambda \to 1$ with appropriate $E_j^{\text{intervene}}$.

\subsection{Counterfactuals}

Counterfactual reasoning proceeds through three steps using energy landscapes rather than probability distributions.

\paragraph{Abduction.} Given conditioning values for subset $\mathcal{O} \subseteq \{1,\dots,n\}$, recover latent configuration and exogenous factors by solving
\[
    (\hat{\mathbf{z}}, \hat{\mathbf{u}}) \;=\; \arg\min_{\mathbf{z}, \mathbf{u}} E(\mathbf{z}, \mathbf{u})
    \quad \text{s.t.} \quad z_k = z_k^{\text{cond}} \; \; \forall k \in \mathcal{O},
\]
where conditioning constraints are implemented as infinite barriers in $E(\mathbf{z}, \mathbf{u})$. Conditioning values may originate from encoded real-world data or represent entirely novel latent configurations conjectured by the system.

\paragraph{Intervention.} For hard intervention $do(z_j = z_j^*)$, replace local potential $E_j(z_j \mid z_{\mathrm{PA}_j}, u_j)$ with infinite barrier outside $z_j^*$, removing influence of both parents and exogenous factor $U_j$. For soft intervention, alter parameters within $E_j$ without clamping, preserving dependencies. Global terms are retained but evaluated with intervened values. Crucially, exogenous configuration $\hat{\mathbf{u}}$ from abduction is held fixed.

\paragraph{Prediction.} Re-minimize modified energy $E'(\mathbf{z}, \hat{\mathbf{u}})$ over intervened variables and descendants, holding other coordinates at abducted values $\hat{z}_k$ and keeping exogenous factors at $\hat{\mathbf{u}}$. This ensures changes propagate only along permissible causal pathways while preserving underlying noise structure.

In geometric view, abduction identifies point $(\hat{\mathbf{z}}, \hat{\mathbf{u}})$ on manifold of admissible states consistent with conditioning. Intervention deforms latent manifold locally while preserving exogenous configuration. Prediction finds new equilibrium on deformed manifold.

\subsection{Disjunctive Interventions}
Disjunctive interventions arise when an agent constrains a latent causal variable to lie in a prescribed set rather than at a single value. For a target variable $Z_j$ and a finite admissible set $\mathcal{S}=\{s_1,\dots,s_m\}$ of latent values, the disjunctive action is denoted $do\big(Z_j\in\mathcal{S}\big)$. The classical structural account is intentionally policy--noncommittal; it gives precise semantics for each singleton $do(Z_j=s)$ by graph surgery, and for $do(Z_j\in\mathcal{S})$ it returns the family $\{do(Z_j=s)\}_{s\in\mathcal{S}}$ (and thus interval bounds) unless a selection rule is supplied. By contrast, the imaging approach \citep{pearl2017physical} fixes a Bayesian-like, odds-preserving redistributive rule: relative to a similarity relation (``closest worlds'') and within each context, it preserves the prior ratios $P(Z_j=s_i\mid \text{context})$ while shifting mass to the corresponding $s_i$--worlds, thereby producing a single mixture effect for the disjunction. Energy-structured causal models (E\textendash SCMs) contain both views as limits: they retain mechanistic locality, support policy-free bounds, and admit an explicit (optional) selection rule via a control energy when a single value is desired.

Let $E(\mathbf{z},\mathbf{u})=\sum_i E_i(z_i\mid z_{\mathrm{PA}(i)},u_i)+\sum_i E_{U_i}(u_i)+E_{\mathrm{global}}(\mathbf{z},\mathbf{u})$ denote the total energy of a static E\textendash SCM. Counterfactual queries proceed by abduction, intervention, and prediction. Abduction yields $(\hat{\mathbf{z}},\hat{\mathbf{u}})$ consistent with conditioning constraints. Intervention modifies only the local terms required by surgery, while $E_{\mathrm{global}}$ is retained and evaluated at the intervened values. Prediction re--minimizes over the descendants of the intervened nodes, holding non--descendants at their abducted values and keeping the exogenous configuration $\hat{\mathbf{u}}$ fixed. All disjunctive statements below are taken relative to this abducted context.

The policy free semantics of $do(Z_j\in\mathcal{S})$ is set--valued surgery. One forms the family of singleton interventions $\{E^{do(Z_j=s)}\}_{s\in\mathcal{S}}$ and refrains from collapsing them to a single composite energy. Any readout $\Phi$ of interest is then bounded at the abducted context by
\begin{equation}
\Phi^{\min} = \min_{s\in\mathcal{S}} \Phi^{do(Z_j=s)}(\hat{\mathbf{z}},\hat{\mathbf{u}}), \qquad
\Phi^{\max} = \max_{s\in\mathcal{S}} \Phi^{do(Z_j=s)}(\hat{\mathbf{z}},\hat{\mathbf{u}}).
\end{equation}
This interval constitutes the structural envelope and expresses that, without a rule for selecting among $\mathcal{S}$, a single number is not identified. When $|\mathcal{S}|=2$, the effect of $do(Z_j\in\{s_2,s_3\})$ lies between the two singleton effects, and the width of the interval quantifies policy sensitivity.

A pointwise minimum across singletons represents the adoption of a specific selection policy rather than the policy free baseline. If one wishes to commit to a mechanism level rule while remaining in energy space, a control energy $R_j(s;\hat{\mathbf{z}},\hat{\mathbf{u}})$ can be introduced to score the admissible choices in the abducted context, with weight $\rho\ge 0$. The composite operator is then
\begin{equation}
E^{do(Z_j\in\mathcal{S})}(\mathbf{z},\mathbf{u}) = \min_{s\in\mathcal{S}}\left\{E^{do(Z_j=s)}(\mathbf{z},\mathbf{u})+\rho R_j\big(s;\hat{\mathbf{z}},\hat{\mathbf{u}}\big)\right\}.
\label{eq:control_min}
\end{equation}
This choice is local and modular, does not require probabilities, and encodes physical preferences, actuator costs, or design constraints. A smooth variant replaces the hard minimum by the log--sum--exp aggregator
\begin{equation}
\mathcal{E}_\tau(\mathbf{z},\mathbf{u}) = -\tau\log\sum_{s\in\mathcal{S}}\exp\left( -\frac{1}{\tau}\left[E^{do(Z_j=s)}(\mathbf{z},\mathbf{u})+\rho R_j(s;\hat{\mathbf{z}},\hat{\mathbf{u}})\right]\right), \qquad \tau>0,
\label{eq:softmin}
\end{equation}
where $\tau\to 0$ recovers \eqref{eq:control_min} and finite $\tau$ yields a mechanism--averaged effect governed by $R_j$. If one subsequently chooses to read $e^{-R_j(s;\hat{\mathbf{z}},\hat{\mathbf{u}})}$, up to normalization and within each relevant context, as a reference selection policy, then \eqref{eq:softmin} induces an imaging--like mixture for the disjunction in the sense of \citep{pearl2017physical}.

\section{LAP Enforcement via Penalties}
\label{sec:LAP}

The Locality--Autonomy Principle (LAP) requires that non-descendants do not influence a module's mechanism or its parameters. In an energy-structured model this means that, in an adapted chart, the stationarity condition for $z_i$ does not depend on $z_A$ or on $\theta_A$ when $A$ is not a descendant of $i$. Formally,
\[
\frac{\partial^2 \widetilde{E}_i^{(A)}}{\partial z_i\,\partial z_A} = 0,
\qquad
\frac{\partial^2 \widetilde{E}_i^{(A)}}{\partial z_i\,\partial \theta_A} = 0,
\qquad
i\notin\mathrm{Desc}(A).
\]
These conditions express that changes in $z_A$ or in the parameters of $A$ do not alter the stationarity condition determining $z_i$. Additive terms independent of $z_i$ are therefore irrelevant to mechanism identity and do not constitute a violation.

\paragraph{Adapted chart (standing assumption).}
Throughout we work in coordinates $(z,u,\theta)$ that are adapted to the causal structure (e.g. DAG) of the model: variables decompose by module $i$, parent masks restrict $E_i$ to $(z_i,z_{\mathrm{PA}(i)},u_i;\theta_i)$, and the flow of $A$ acts only on $A$ and its descendants. In this chart the LAP tests
$\partial^2_{z_i z_A}\widetilde E_i^{(A)}=0$ and $\partial^2_{z_i \theta_A}\widetilde E_i^{(A)}=0$ (for $i\notin\mathrm{Desc}(A)$) are equivalent to the coordinate-free conditions $\mathcal L_{\xi_A}\mathcal M_i=0$ and $\mathcal L_{\Xi_A}\mathcal M_i=0$. When alternative latent parameterizations are used (e.g., via an encoder), tests are applied after re-expressing the model in its induced adapted chart.

\paragraph{Identifying relevant couplings.}
When a global term $E_{\mathrm{global}}$ couples multiple modules, only the parts that introduce dependence between $A$ and $i$ should be penalized. For each ordered pair $(A,i)$ with $i\in\mathrm{NonDesc}(A)$, isolate the portion of $E_{\mathrm{global}}$ that depends jointly on variables associated with $A$ and with $i$:
\[
E_{\mathrm{global}}^{(A,i)}(\mathbf z_{S_{A,i}},\mathbf u_{S_{A,i}}) :=
E_{\mathrm{global}}(\mathbf z,\mathbf u)\big|_{\mathbf z_j,\mathbf u_j\text{ fixed for }j\notin S_{A,i}},
\]
where $S_{A,i}$ is the set of coordinates that mediate this coupling. The effective energy for module $i$ relative to $A$ is then
\[
\widetilde{E}_i^{(A)} = E_i(z_i\mid z_{\mathrm{PA}(i)},u_i) + E_{U_i}(u_i) + E_{\mathrm{global}}^{(A,i)}.
\]

\paragraph{Pointwise LAP condition.}
For each non-descendant pair $(A,i)$,
\[
\partial^2_{z_i z_A} \widetilde E_i^{(A)} = 0,
\qquad
\partial^2_{z_i \theta_A} \widetilde E_i^{(A)} = 0,
\]
evaluated at equilibria or along trajectories of interest. These equalities define the LAP structurally, without averaging.

\paragraph{Training penalties.}
During learning, these pointwise conditions can be approximated by penalties computed over sampled states:
\[
\mathcal{L}_{\text{LAP}}^{\text{static}}
=
\sum_{A=1}^{n}\sum_{i\in\mathrm{NonDesc}(A)}
\lambda_{A,i}\,
\mathbb{E}_{(\mathbf z,\mathbf u)\sim\mathcal Q}
\!\left[\left\|\frac{\partial^2 \widetilde E_i^{(A)}}{\partial z_i\,\partial z_A}\right\|^2\right]
+
\mu_{A,i}\,
\mathbb{E}_{(\mathbf z,\mathbf u)\sim\mathcal Q}
\!\left[\left\|\frac{\partial^2 \widetilde E_i^{(A)}}{\partial z_i\,\partial \theta_A}\right\|^2\right].
\]
Here $\mathcal Q$ is a sampling distribution over $(\mathbf z,\mathbf u)$, for example equilibria reached during training. Expectations serve only as empirical surrogates for the universal structural conditions.

\paragraph{Dynamic formulation.}
In the dynamic setting, mechanisms are specified by a vector field $F=(F_1,\dots,F_n)$ with
\[\dot z_i = F_i\bigl(z_i, z_{\mathrm{PA}(i)}, u_i;\, \theta_i\bigr).\]
Let $\xi_A$ denote the state--flow generated by variations of $z_A$ (with parameters fixed), and let $\Xi_A$ denote the parameter--flow generated by variations of $\theta_A$ (with states fixed). The Locality--Autonomy Principle (LAP) requires that non--descendants neither influence the mechanism for $i$ through state flows nor through parameter flows. There are two equivalent ways to express this.

\paragraph{Component (directional--derivative) form.} In an adapted chart where $\xi_A = \partial/\partial z_A$, the LAP conditions for non--descendants $i\notin \mathrm{Desc}(A)$ are
\begin{equation}
\label{eq:lap-dyn-component}
\frac{\partial F_i}{\partial z_A} = 0,\qquad \frac{\partial F_i}{\partial \theta_A} = 0.
\end{equation}
These equalities state that neither changes in $z_A$ nor changes in the parameters of module $A$ affect the right-hand side that determines $\dot z_i$.

\paragraph{Vector--field (Lie--derivative/commutator) form.} If one bundles the $i$th mechanism as the vector field $F_i \, e_i$ (with $e_i$ the $i$th coordinate basis vector), then LAP can be written as vanishing Lie derivatives along the flows of $A$:
\begin{equation}
\label{eq:lap-dyn-commutator}
\mathcal L_{\xi_A}(F_i e_i) = [\xi_A,\,F_i e_i] = 0,\qquad \mathcal L_{\Xi_A}(F_i e_i) = [\Xi_A,\,F_i e_i] = 0,
\end{equation}
for all $i\notin \mathrm{Desc}(A)$. In adapted coordinates where $e_i$ is constant and $\xi_A = \partial/\partial z_A$, the commutator reduces to
\[
[\partial_{z_A},\,F_i e_i] = (\partial F_i/\partial z_A)\, e_i,
\]
so \eqref{eq:lap-dyn-commutator} is equivalent to the component conditions \eqref{eq:lap-dyn-component}. The same argument applied to the parameter flow $\Xi_A$ yields the equivalence of $\mathcal L_{\Xi_A}(F_i e_i)=0$ and $\partial F_i/\partial\theta_A=0$.

\paragraph{Evaluation points.} These conditions are to be checked pointwise along trajectories or at equilibria of interest. When some coordinates are clamped or eliminated, use the effective dynamics obtained after substitution or Schur complementation, and apply \eqref{eq:lap-dyn-component}--\eqref{eq:lap-dyn-commutator} to the reduced system.

\section{ICM Enforcement via Penalties}
\label{sec:ICM}

The Independent Causal Mechanisms (ICM) principle complements the Locality--Autonomy Principle (LAP). 
LAP enforces \emph{graph-locality}: non-descendants may not influence a module's mechanism, either through their states or their parameters. 
ICM enforces \emph{parameter-space separability}: parent parameters do not alter the form of a child mechanism, and the parameter families of parent and child admit a local product structure. 
Together, LAP and ICM ensure that causal independence holds both in the flow of values and in the structure of mechanisms.

\paragraph{Static formulation.}
Define the residual
\[
G_i(z,u;\theta)\;:=\;\frac{\partial E(z,u;\theta)}{\partial z_i}.
\]
The mechanism for node $i$ is given implicitly by the stationarity equation $G_i(z,u;\theta)=0$ (equivalently, the Euler–Lagrange condition for $z_i$). In an adapted chart, ICM requires
\[
\frac{\partial G_i}{\partial \theta_{\mathrm{PA}(i)}} = 0,
\qquad
\frac{\partial^2 G_i}{\partial \theta_{\mathrm{PA}(i)}\,\partial \theta_i} = 0.
\]
The first condition, evaluated in an adapted chart with parent states 
$z_{\mathrm{PA}(i)}$ held fixed and using the effective local energy for 
node $i$, ensures that upstream parameters do not alter the stationarity 
equation determining $z_i$. 
The second expresses local separability: it requires vanishing mixed parent–child parameter interaction, meaning the parameter flows of parent and child commute so their parameter families form a local product structure.

\paragraph{Dynamic formulation.}
For dynamic E\textendash SCMs defined by $\dot z_i = F_i(z_i, z_{\mathrm{PA}(i)}, u_i; \theta_i)$, the same structure holds:
\[
\frac{\partial F_i}{\partial \theta_{\mathrm{PA}(i)}} = 0,
\qquad
\frac{\partial^2 F_i}{\partial \theta_{\mathrm{PA}(i)}\,\partial \theta_i} = 0.
\]
The first condition enforces structural independence of the evolution law from upstream parameters; the second ensures separability of parameter flows. 
When these hold, the dynamics of each mechanism depend only on its own parameters and those of its exogenous inputs.

\paragraph{Training penalties.}
The differential conditions can be approximated during learning by regularization penalties computed over sampled equilibria or trajectories:
\[
\mathcal{L}_{\text{ICM}}^{\text{static}}
=
\sum_{i=1}^{n}
\alpha_i\,\mathbb{E}_{(\mathbf z,\mathbf u)\sim\mathcal Q}
\!\left[\left\|\frac{\partial G_i}{\partial \theta_{\mathrm{PA}(i)}}\right\|^2\right]
+
\beta_i\,\mathbb{E}_{(\mathbf z,\mathbf u)\sim\mathcal Q}
\!\left[\left\|\frac{\partial^2 G_i}{\partial \theta_{\mathrm{PA}(i)}\,\partial \theta_i}\right\|^2\right],
\]
\[
\mathcal{L}_{\text{ICM}}^{\text{dynamic}}
=
\sum_{i=1}^{n}
\alpha_i\,\mathbb{E}_{(\mathbf z,\mathbf u)\sim\mathcal Q}
\!\left[\left\|\frac{\partial F_i}{\partial \theta_{\mathrm{PA}(i)}}\right\|^2\right]
+
\beta_i\,\mathbb{E}_{(\mathbf z,\mathbf u)\sim\mathcal Q}
\!\left[\left\|\frac{\partial^2 F_i}{\partial \theta_{\mathrm{PA}(i)}\,\partial \theta_i}\right\|^2\right].
\]
Here $\mathcal Q$ denotes a sampling distribution over $(\mathbf z,\mathbf u)$, typically drawn from equilibria or simulated dynamics. 
The expectations are empirical surrogates; the structural definition itself is pointwise.

\section{Geometric Structure}
\label{sec:GeometricStructure}
\subsection{Fractured Entangled Representations}

The fractured–entangled representation (FER) hypothesis \citep{kumar2025questioning} challenges representational optimism in deep learning. Even when the data-generating process admits a clean, factorizable representation, empirical risk minimization can converge to latents that are entangled, scrambled, and fractured across the manifold. These pathologies are largely invisible when judged only on the observational distribution $P_{\mathrm{obs}}$: a network may appear competent on downstream tasks yet fail to respect the causal or modular structure of the underlying process.

Gauge symmetry provides language for analyzing this phenomenon. Learned representations fall into equivalence classes under reparametrizations that keep designated outputs fixed. There are two natural choices of invariants. If the invariants are observational heads, one obtains a large observational gauge within which many representations are indistinguishable. If the invariants are causal heads, one obtains a smaller causal gauge. FER arises in the gap: two representations can be equivalent observationally while disagreeing interventionally.

\subsection{Gauge Symmetry}

Let $H_{\mathrm{obs}}$ denote the chosen observational heads and $H_{\mathrm{causal}}$ the chosen causal heads. A transformation preserves $H_{\mathrm{obs}}$ if it leaves all values reported by $H_{\mathrm{obs}}$ unchanged; it preserves $H_{\mathrm{causal}}$ if it leaves all values reported by $H_{\mathrm{causal}}$ unchanged. Consider an encoder $f:\mathcal{X}\to\mathcal{Z}$ producing observable latents $z=f(x)$, and an energy $E(z)$ defined over observable variables $z$. Write
\[
\gamma:(f,E)\mapsto\big(\gamma_f\!\circ f,\; E\!\circ\gamma_z^{-1}+c\big),
\]
where $\gamma_f,\gamma_z$ are invertible reparametrizations and $c\in\mathbb{R}$ (or per–module $c=(c_i)_i$) is an additive energy offset. Define
\[
\Gamma_{\mathrm{obs}}=\{\gamma:\ H_{\mathrm{obs}}(f,E)=H_{\mathrm{obs}}(\gamma_f\!\circ f,\;E\!\circ\gamma_z^{-1}+c)\},
\]
\[
\Gamma_{\mathrm{causal}}=\{\gamma:\ H_{\mathrm{causal}}(f,E)=H_{\mathrm{causal}}(\gamma_f\!\circ f,\;E\!\circ\gamma_z^{-1}+c)\}.
\]
Both sets are groups under composition, with identity and inverses, and $\Gamma_{\mathrm{causal}}\subseteq\Gamma_{\mathrm{obs}}$ whenever the causal heads refine the observational heads. Gauge orbits are the equivalence classes induced by these actions. Movement along an observational orbit leaves $H_{\mathrm{obs}}$ invariant but can distort geometry, factorization, and causal alignment. Movement along a causal orbit leaves $H_{\mathrm{causal}}$ invariant and is therefore compatible with intervention semantics. FER reflects the existence of elements in $\Gamma_{\mathrm{obs}}\setminus\Gamma_{\mathrm{causal}}$: representations that are observationally identical yet causally inequivalent.

Let $M$ denote the manifold of pairs $(f,E)$. The action of a chosen gauge partitions $M$ into orbits. At any $(f,E)\in M$, the tangent space splits into vertical directions generated by the Lie algebra of the gauge (which move within an orbit and change only coordinates) and horizontal directions that alter mechanisms themselves. When training is driven by $H_{\mathrm{obs}}$ alone, empirical risk minimization provides no inherent pressure to prefer horizontal motion over vertical, so optimization can wander extensively within observational orbits.

\paragraph{Observable sets and gauge groups}

The size of a gauge depends on what is observed. Let $H=\{g_1,\dots,g_m\}$ denote a chosen set of heads. The associated gauge is
\[
\mathcal{G}(H)=\{\gamma\in\Gamma:\ g_k(\gamma\!\cdot\!(f,E))=g_k(f,E)\ \text{for all}\ k\}.
\]
Richer heads shrink $\mathcal{G}(H)$ by making more aspects of $(f,E)$ observable. Reporting absolute per-module energies in fixed coordinates on an open set fixes offsets, scales, and generically latent reparametrizations. Reporting only derivatives or energy differences removes additive constants but not per-module scales unless an external calibration is imposed. Gradients and differences can certify modular structure yet still admit scale ambiguity. Causal heads reduce representational redundancy further, because post-surgery equilibria, intervention energy differences and Hessians, and causal gradients computed under interventions constrain admissible reparametrizations while leaving unit choices unconstrained unless physically fixed. Independence tests in an adapted chart, expressed as vanishing cross-partials of effective energies for non-descendants, instantiate locality and autonomy in differential form and further reduce the gauge. As causal heads are introduced, more of the observational gauge is broken and the two groups approach one another. 

\paragraph{Causal queries define observables}

Causal content is carried by specific queries to the energy landscape under abduction–intervention–prediction. Interventional equilibria after local surgery, energy differences and Hessians on the post-surgery landscape, and causal gradients under interventions distinguish coordinate changes that merely relabel states from edits that alter mechanisms. Counterfactual evaluations that hold the abducted exogenous configuration fixed add further constraints. These queries progressively reduce admissible reparametrizations while preserving any remaining unit ambiguity unless an external calibration is supplied.

\paragraph{Hierarchy of observable sets}

Different head families constrain gauge freedom to different extents. Writing $H_0$ for a minimal noninformative set, the associated gauge groups satisfy
\begin{equation}
\mathcal G(H_0)
  \;\supset\;\mathcal G(H_E)
  \;\supset\;\mathcal G(H_{\partial E})
  \;\supset\;\mathcal G(H_{\nabla E})
  \;\supset\;\mathcal G(H_{\Delta E})
  \;\supset\;\mathcal G(H_{\mathrm{Hess}}),
\end{equation}
where $H_E$ returns absolute energies, $H_{\partial E}$ returns partials, $H_{\nabla E}$ returns full gradients, $H_{\Delta E}$ returns energy differences, and $H_{\mathrm{Hess}}$ returns Hessians, all as numeric values in fixed coordinates on an open set. The table summarizes identifiability under these heads.

\begin{center}
\small
\begin{tabular}{@{}p{1.8cm}p{6.6cm}p{6.4cm}@{}}
\toprule
Head set $H$ & Newly ruled–out transformations & Survivors \\
\midrule
$H_E$ & Per–module constant shifts $E_i\!\mapsto\!E_i\!+\!c_i$; per–module scalings $E_i\!\mapsto\!s_iE_i$ & Only reparametrizations that leave each $E_i$ unchanged (up to coordinate permutations between identical modules) \\[0.5em]
$H_{\partial E}$ & Per–module scalings if magnitudes are recorded (derivatives scale by $s_i$) & Global constant shifts; scales if only directions are recorded \\[0.5em]
$H_{\nabla E}$ & Coordinated scalings (gradient magnitudes fix units) & Global constant shifts \\[0.5em]
$H_{\Delta E}$ & Coordinated scalings (differences fix units) & Global constant shifts \\[0.5em]
$H_{\mathrm{Hess}}$ & None beyond those fixed by second–order structure & Global constant shifts unless absolute energies are also observed \\
\bottomrule
\end{tabular}
\end{center}

This dual-gauge view resolves the central tension introduced above: a network can be observationally perfect yet causally brittle. Gauge invariance with respect to causal quantities is the desired end state in which interventions are stable and basis choices become immaterial for counterfactual competence. Gauge invariance with respect to observational quantities alone explains why fractured and entangled coordinates persist without harming surface performance. The difference between the two gauges is the room in which FER lives. Reducing that room by adding causal measurements, constraints, and diagnostics is a form of error correction: it converts previously invisible representational choices into observable commitments and aligns the learned coordinates with mechanism-level semantics. See Appendix \ref{sec:appendix-hierarchy} for more detail.

\subsection{Riemannian Geometry of Equilibria}

At equilibria defined by $\nabla_{z} E(z,\theta)=0$, write block Hessians
\[
H_{zz} := \partial^2_{z z} E,\qquad H_{z\theta} := \partial^2_{z\theta} E,\qquad H_{\theta\theta} := \partial^2_{\theta\theta} E.
\]

At a stable equilibrium $(z^*,\theta)$ with $\nabla_z E(z^*,\theta)=0$ and $H_{zz}(z^*,\theta)\succ 0$, the quadratic form 
$\delta z^\top H_{zz}(z^*,\theta)\,\delta z$ is positive definite. 
Under a reparametrization $z=\phi(\zeta)$ with $J=\partial z/\partial\zeta|_{\zeta^*}$,
\[
H_{\zeta\zeta}(\zeta^*,\theta)=J^\top H_{zz}(z^*,\theta)J,
\]
so these forms define a local Riemannian metric (energetic stiffness). 
Away from critical points the coordinate Hessian is not tensorial.


\begin{proposition}[Causal metric at equilibria]
Fix $\theta$ and stable equilibrium $z^*$ with $\nabla_z E(z^*,\theta)=0$ and $H_{zz}(z^*,\theta)\succ 0$. Then
\begin{enumerate}[label=(\roman*),leftmargin=*]
  \item (Local inner product) $\langle u,v\rangle_{z^*} := u^{\top} H_{zz}(z^*,\theta) v$ is symmetric, positive-definite bilinear form on latent tangent space, defining Riemannian metric $g_{zz}$.
  \item (Coordinate invariance at critical points) For smooth reparametrization $z=\phi(\zeta)$ with Jacobian $J=\partial z/\partial\zeta$, Hessian transforms as $H_{\zeta\zeta}= J^{\top} H_{zz} J$. Hence quadratic energy change $\tfrac12\,\delta\zeta^{\top} H_{\zeta\zeta}\,\delta\zeta$ equals $\tfrac12\,\delta z^{\top} H_{zz}\,\delta z$, i.e., chart-independent at equilibria.
\end{enumerate}
\end{proposition}

\paragraph{Gauge/units note.}
Additive energy offsets do not affect $H_{zz}$, but per–module energy rescalings
$E_i \mapsto a_i E_i$ multiply the corresponding blocks of $H_{zz}$ by $a_i$. Thus the
Riemannian metric is defined only up to such positive scales unless units are calibrated. Structure that is invariant to these scales,
e.g., sparsity patterns, zero cross–blocks implied by LAP, and angle- or correlation-based comparisons after per-block normalization, remains meaningful and gauge-robust.

\paragraph{Effective Hessians under conditioning.}
When equilibria are computed with some coordinates clamped or marginalized (e.g., holding
$u$ or $z_{\mathrm{int}}$ fixed after abduction), the relevant curvature for local responses on
$z$ is the Schur complement of the full Hessian with respect to the free coordinates. All
statements above then apply with $H_{zz}$ replaced by this effective Hessian.

\begin{corollary}[Local susceptibility / implicit-function response]
Let $F(z,\theta):=\nabla_z E(z,\theta)$. If $H_{zz}(z^*,\theta)$ is nonsingular, implicit function theorem yields smooth map $\theta\mapsto z^*(\theta)$ with first-order response
\[
\frac{\partial z^*}{\partial w}\;=\; -\, H_{zz}^{-1}\, \partial^2_{z w} E,\qquad w\in\{z_A,\theta_A\}.
\]
If LAP holds in adapted chart, then for non-descendant $i\notin\mathrm{Desc}(A)$, cross-partials $\partial^2_{z_i z_A}E$ and $\partial^2_{z_i\theta_A}E$ vanish, yielding $\partial z^*_i/\partial w = 0$.
\end{corollary}

\paragraph{Diagnostics and uses.}
\begin{enumerate}[label=(\alph*),leftmargin=*]
  \item \textbf{Modularity test.} In adapted chart, check cross-blocks $(H_{zz})_{iA}$ for $i\notin\mathrm{Desc}(A)$; small norms witness LAP.
  \item \textbf{Energetic length for interventions.} Local line element $\mathrm{d}\ell^2 = \mathrm{d}z^{\top} H_{zz}\,\mathrm{d}z$ quantifies perturbation ``size'' in energetic units, not Euclidean distance.
  \item \textbf{Fragility/identifiability.} Condition number $\kappa(H_{zz})$ flags ill-posed regimes where compliance $H_{zz}^{-1}$ is large and responses blow up.
  \item \textbf{Algorithmic connection.} Steepest descent under causal metric is Newton/Riemannian step $s^{\star} \;=\; -\, H_{zz}^{-1}\,\nabla_z E$, minimizing local quadratic model under energetic line element.
\end{enumerate}

Important caveats: (i) At saddles, $H_{zz}$ is indefinite; restrict to stable subspace or regularize. (ii) Away from equilibria, coordinate Hessian is not tensorial; use covariant Hessian $\mathrm{Hess}_g E$ induced by background metric.

\subsection{Latent Steps in Abduction\textendash Intervention\textendash Prediction}

In E\textendash SCMs, algorithmic work often occurs in latent variables $z$, not only parameters $\theta$. A \emph{latent step} is iterative update
\[
  z \leftarrow z + s,\qquad s \in T_z Z,
\]
used to solve for equilibrium $z^*$ under energetic objective. Principal contexts:

\begin{enumerate}[label=(\roman*),leftmargin=*]
  \item \textbf{Abduction.} Given observed coordinates clamped, recover $(\hat z,\hat u)$ via
  \[
  (\hat z,\hat u)\;=\;\arg\min_{z,u} E(z,u;\theta)\quad\text{subject to observation constraints,}
  \]
  which requires latent steps until stationarity $\nabla_z E=0$.

  \item \textbf{Prediction after hard intervention.} After $do(Z_A=z_A^{\circ})$, non-descendants remain clamped while descendants re-equilibrate:
  \[
  z_{\mathrm{desc}}\;=\;\arg\min E\big( z_{\mathrm{nondesc}}{=}\hat z,\; z_{\mathrm{desc}},\; u{=}\hat u;\,\theta\,\big|\,do\big),
  \]
  requiring latent steps for unclamped coordinates.

  \item \textbf{Soft interventions.} Modifying local potential $E_A$ without clamping changes landscape and necessitates re-solving for new equilibrium.

  \item \textbf{Dynamic E\textendash SCMs.} If inference realizes as gradient flow $\dot z = -\nabla_z E$ or preconditioned variant, time-discretization produces latent step sequences converging to equilibrium.
\end{enumerate}

\subsection{Computational Considerations}

The Newton/Riemannian step in latent space
\[
  s^{\star} \;=\; -\, H_{zz}^{-1}\,\nabla_z E,
\]
is practical in E\textendash SCMs because (i) $H_{zz}$ is taken over latents (typically $\dim z \ll \dim\theta$), and (ii) under LAP/ICM it is often modular (block/sparse), so solves are cheap.

For handling infinite barriers in practice, use penalty methods or projected gradient descent. For numerical stability with Hessian-based methods, regularize as $H_{zz}+\lambda I$ when near singular configurations. Choice of optimization algorithms depends on problem structure: gradient descent for simple landscapes, Newton methods when Hessian is well-conditioned and cheap to compute, quasi-Newton (L-BFGS) for intermediate cases.

\subsection{Gauge freedom and causality in LLMs}
Intervening in large language models is conceptually delicate because these systems do not expose explicit mechanisms. In the absence of identified mechanisms, the term intervention can only be used heuristically, as a diagnostic edit to internal coordinates rather than a principled causal operation. With this caveat, we analyze gauge freedom in LLMs and explain how it undermines interventional uniqueness unless edits are defined in a gauge-aware manner.

Large language models exhibit the same representational slack that motivates the gauge perspective in E\textendash SCMs. For a frozen network, observational behavior is determined by the logits it produces for every input, and many internal reparameterizations leave those logits unchanged. 

Let $h=\sigma(W_1 x)$ and $y=W_2 h$ denote a simplified feedforward block within a transformer layer. For any invertible matrix $A$, the reparameterization that transforms the hidden representation as $h' = Ah$ while setting $W_2' = W_2 A^{-1}$ (keeping $W_1$ unchanged) preserves the mapping $x \mapsto y$ on all inputs: $y = W_2 h = W_2 A^{-1}(Ah) = W_2' h'$. This transformation traces out observationally equivalent configurations that differ in their internal coordinatization of the post-nonlinearity representation.

Likewise, in attention one can transform the value space by $V' = V A$ with a compensating output map $W_O' = W_O A^{-1}$, or apply joint changes to queries and keys of the form $Q' = Q R^{-\top}$ and $K' = K R$, keeping $Q' K'^\top = Q K^\top$ and therefore the attention weights invariant. These transformations trace out orbits of observationally equivalent models; they are the gauge degrees of freedom in plain neural architectures. In practice, LLMs primarily approximate conditional distributions over tokens and do not encode mechanism-level structure, so coordinate-wise edits at hidden layers do not by themselves define stable causal operations.

Although such models are equal observationally, they are not generally equal under intervention unless the intervention operator co-transforms with the gauge. Consider a coordinate clamp in the hidden representation. In the original basis, an edit of the form $do(h_k := c)$ specifies a directional operation tied to the $k$-th axis. After transformation $h' = Ah$ with $W_2' = W_2A^{-1}$, a coordinate based intervention changes its meaning. Clamping the $k$-th coordinate in each basis, $(h)_k = c$ versus $(h')_k = c$, implements different physical edits, since these constrain different directions: $e_k$ in the original space versus $A^{-1}e_k$ (the $k$-th coordinate of $h'$ in original coordinates). The post-intervention prediction therefore depends on the chosen gauge. 

The same phenomenon appears in attention: deleting one head or nudging a value feature in the original basis corresponds, after a compensating value-output transformation, to a mixture of heads or features in the new basis.
Specifically, if $V' = VA$ and $W_O' = W_O A^{-1}$, then an intervention 
targeting the $k$-th value feature in the original basis corresponds to 
intervening on $\sum_j (A^{-1})_{kj} v'_j$ in the transformed basis, a 
weighted combination across multiple features. If the deletion is specified naively in the new coordinates, the resulting counterfactual behavior can change even though all unedited forward evaluations agree.

This distinction helps to clarify why fractured and entangled representations persist in practice without harming surface performance yet undermine causality-aware prediction. Empirical risk minimization constrains outputs but is largely indifferent to the internal coordinatization that mediates edits. The training objective allows the model to settle anywhere along a gauge orbit, including coordinates that mix distinct semantic factors or split a single factor across disconnected fragments. Probes that read linear relationships in a single basis, or interventions that clamp raw coordinates, are therefore gauge dependent: a rotation or rescaling that is observationally invisible can invalidate the intended semantics of the probe or the edit.

E\textendash SCMs address the ambiguity by making intervention semantics part of the model. Interventions are defined as local surgeries on constraint energies rather than as coordinate edits of intermediate activations. Surgery is stated in terms of mechanisms and their parent sets, so the meaning of an edit is canonical and does not depend on a particular latent basis. This fixes what is edited, but it does not by itself determine how the learned representation is coordinatized. The same semantics can be implemented by gauge-related encoder–energy pairs whose latent coordinates differ. To align the representation with the semantics, we describe an observable probe hierarchy (and LAP penalties) that reduces gauge freedom by turning previously invisible representational choices into measurable commitments. Absolute energy readouts eliminate additive and scaling ambiguities, derivative and Hessian probes reveal cross-partials indicating nonlocal influence, and LAP-aligned penalties suppress such cross-effects for non-descendants. As gauge freedom is reduced in this way, interventional predictions become stable under reparameterization, and the learned coordinates more faithfully express modular structure. This contrasts with many existing models, where interventions remain heuristic and gauge-dependent unless additional structure is imposed.

\section{Explanations of Observations via E\textendash SCMs}
\label{sec:Explanations}
The preceding sections defined the Energy--Structured Causal Model (E\textendash SCM) as an energy function
\[
E:(z,u)\mapsto \mathbb{R},
\]
whose minima or equilibrium points define admissible configurations of endogenous variables \(z\) and exogenous conditions \(u\).
This formulation extends the structure of classical Structural Causal Models by replacing explicit functional dependencies with an implicit equilibrium semantics defined by the energy landscape.
Interventions act by editing \(E\), producing a new energy \(E^{I}\) whose equilibria define the outcomes of the intervention.

While the pair \((z,u)\) captures the internal mechanics of an E\textendash SCM, it is useful to reinterpret the same structure from the perspective of explanation.
In this view, an E\textendash SCM expresses how non--observables account for the observable data.
This section formalizes that explanatory perspective and clarifies its relation to the mechanistic definition above.

\paragraph{Observables, Latents, and Structural Parameters.}
Let $O$ denote observable quantities, 
$L$ denote latent variables (observable in principle), and
$\Theta$ denote non--observable structural parameters.

The triplet \((O,L;\Theta)\) corresponds to different levels of accessibility:
observables are measured or measurable,
latents are unmeasured internal states that could in principle be observed,
and structural parameters define the causal mechanisms but are not themselves observable.

The energy \(E(z,u)\) can be rewritten without loss of generality as
\[
E(O,L;\Theta),
\]
by mapping endogenous variables \(z\) to latents \(L\) and exogenous conditions \(u\) to the fixed background structure \(\Theta\).
In this formulation, \(E\) defines how particular configurations of non--observables \((L,\Theta)\) make the observables \(O\) consistent with the underlying mechanisms.

\paragraph{Accounting of Observables.}
Given an observed configuration \(o\), the latent witnesses that make it consistent with the structure are given by the \emph{abduction set}
\[
\mathcal{A}(o) = \arg\min_{L} E(o,L;\Theta).
\]
Each \(\ell \in \mathcal{A}(o)\) is a causal accounting (relative to \(E,\Theta\) and a declared surgery policy) up to causal equivalence provided that
(i) \((o,\ell)\) lies in the equilibrium set of \(E\);
(ii) \(E\) couples \(L\) to \(O\) nontrivially, so that the explanation is not vacuous; and
(iii) edited energies \(E^{I}\) yield well-posed counterfactuals with \(L_{\mathrm{hold}}=\ell_{\mathrm{hold}}\) for every admissible intervention \(I\).
Explanations that differ only by transformations preserving all intervention-level predictions are identified as equivalent.

When the set $\mathcal{A}(o)$ is nonempty, the model accounts for the observation.
A deterministic selector \(\hat L(o)\) may be introduced to choose a particular witness, for instance by a minimal--norm or minimal--complexity criterion.
This gives a single configuration \(\hat L(o)\) such that
\[
E(o,\hat L(o);\Theta) = \min_{L} E(o,L;\Theta).
\]
The pair \((E,\hat L(o))\) thus constitutes an explanation of the observation under the structure \(\Theta\).

\paragraph{Interventions and Re--equilibration.}

Let \(E^{I}\) denote the edited energy obtained by applying an intervention \(I\).
As in the mechanistic definition, equilibrium points of \(E^{I}\) define counterfactual outcomes.
To describe interventions in the explanatory formulation, partition the latent variables into
\[
L = L_{\text{hold}} \cup L_{\text{free}},
\]
where \(L_{\text{hold}}\) are held fixed across interventions (operationally exogenous) and \(L_{\text{free}}\) are re--equilibrated (operationally endogenous).
Given the abducted latent state \(\hat L(o)\), counterfactual predictions for any target subset \(Y\subseteq O\) are defined by
\[
\mathcal{P}^{I}_Y(o)
= \pi_Y\!\left(
\arg\min_{O',L_{\text{free}}}
E^{I}\big(O',L_{\text{hold}}=\hat L_{\text{hold}}(o),L_{\text{free}};\Theta\big)
\right),
\]
where \(\pi_Y\) projects to the coordinates of interest and clamped observables are replaced by the values specified by the intervention.
This gives a deterministic counterpart of the abduction--action--prediction scheme found in probabilistic SCMs.

\paragraph{Inferential Codes and Practical Approximation.}

In practice, the minimization defining \(\hat L(o)\) may be approximated by an encoder network
\[
z = \mathrm{enc}_\phi(o), \qquad
\hat L(o) \approx S_\psi(o,z),
\]
where \(S_\psi\) performs a local optimization or deterministic selection conditioned on the code \(z\).
This allows amortized inference without altering the causal semantics of \(L\).
The encoder provides an efficient computational approximation to the abducted latent state but does not define the latent variable itself.
The latent variable retains its causal meaning through the equilibrium relation defined by \(E\).

If desired, the encoder can be incorporated directly into the mechanism by introducing a soft constraint
\[
E_{\lambda}(O,L,z;\Theta,\phi)
= E(O,L;\Theta) + \lambda\|L-G(z)\|^2 + R(z),
\quad z=\mathrm{enc}_\phi(O),
\]
where \(\lambda\) controls the strength of coupling between the inferred code and the mechanistic latent.
A hard constraint (\(\lambda\to\infty\)) enforces \(L = G(z(O))\) deterministically, in which case the choice of whether to hold \(z\) fixed or recompute it under interventions determines the exogenous or endogenous status of the code.

\paragraph{Interpretation.}

The explanatory perspective recasts the E\textendash SCM as a mapping from observations to equilibrium--consistent latent configurations.
An explanation, $\mathrm{Ex}$, is an ordered triple
\[
\mathrm{Ex}(O) = (E,\hat L(O),\Theta),
\]
where \(E\) defines the mechanism, \(\hat L(O)\) represents the abducted latent state consistent with the observations, and \(\Theta\) encodes the fixed background structure.
Interventions are expressed as edits to \(E\) and re--equilibration of the latent and observable variables under the specified surgery policy.

This formulation highlights that E\textendash SCMs do not merely describe statistical dependencies but provide causal accountings of observables in terms of non--observables.
The energy function serves as a compact representation of the explanatory structure that remains invariant across interventions, while the latent configuration \(\hat L(O)\) constitutes the specific causal state consistent with the observed data.

\section{Integration with Deep Learning Architectures}
\label{sec:Integration_with_DL}
\paragraph{Design goal.}
E\textendash SCMs supplement rather than replace modern deep neural systems. Neural networks provide the high-capacity substrate on which explanations can be hosted, while the energy formalism supplies a declarative layer that makes mechanisms explicit, maintains global consistency under local edits, and exposes internal commitments for criticism. The interface between the substrate and the causal layer may be organized around three roles: adaptors that extract causal latents from the substrate, mechanisms that impose constraints over those latents, and actuators that perform surgery. Measurement is handled separately by probes (heads) that read what the model already commits to without introducing new assumptions.

\paragraph{Adaptors and mechanisms.}
An adaptor is a feature extractor that takes latent representations from a deep module and returns causal latents suitable for use in an E\textendash SCM. It performs amortized abduction by producing good initializers for equilibrium solves and by aligning the coordinates exposed to the mechanism with the causal graph. A mechanism is then a parametric constraint that, given a latent variable and its parents, returns either a local energy in the static case or a vector field in the dynamic case. Parent structure is enforced both architecturally, by masking inputs to reflect $\mathrm{PA}(i)$, and through LAP penalties that suppress cross-partials from non-descendants. Mechanisms compose additively at the level of energy or as summed vector fields, and optional global terms represent shared constraints that span multiple modules. Neural parameterizations instantiate the mechanism parameters while respecting sparsity implied by the graph.

\paragraph{Actuators and surgery semantics.}
Actuators edit mechanisms locally and provide the implementation of interventions. Hard actions clamp a variable by replacing its local energy with an infinite barrier off the clamped value. Soft actions deform the local energy without deleting edges, representing mechanism shifts or biasing influences. Disjunctive actions constrain a variable to a set and are realized either as a set-valued surgery with bounds or via a control energy that selects among admissible values in context. These operations are defined at the level of constraints and preserve the discipline that non-descendants remain unaffected except through allowed global terms.

\paragraph{Probes as an external measurement interface.}
We use probe synonymously with observable head $H$. Probes are not part of the mechanism tuple. They are an external measurement interface that reads properties of the energy landscape and its equilibria before and after surgery. Absolute and relative energy levels, first and second derivatives, and post-intervention equilibria are examples of such readouts. Probes neither modify the model nor add parameters or gradient updates; they simply make the model's internal commitments observable. Because they are external, they support black-box testing of modularity, interventional consistency, and transfer without entangling verification with implementation.

\paragraph{Equilibrium layers and differentiation.}
Inference proceeds by solving for equilibria. In the static case this is a constrained minimization over latents and exogenous variables consistent with observations; in the dynamic case it is the computation of steady states or controlled trajectories. Differentiation through equilibria uses implicit differentiation when the Hessian with respect to latents is well conditioned, yielding a response map that depends on the inverse Hessian and mixed second derivatives. When conditioning is poor, limited unrolling or regularized solves are preferable. LAP encourages block structure and sparsity in the latent Hessian, which makes conjugate gradients and low-rank preconditioners practical. Because equilibrium solves occur in latent space, the dimensionality is typically modest compared to the parameter space, and costs are compatible with modern training loops.

\paragraph{Analogy and reuse.}
Compositional analogy is realized by reusing mechanism templates across tasks while adapting only the adaptors. A new domain supplies a different (deep, neural) substrate and therefore different upstream latents, but the same family of constraints can be imposed once causal coordinates are instantiated. Small adapters and unit calibrations align coordinates without altering the relational semantics. Probes then evaluate whether the reused mechanisms preserve intervention responses and whether modularity holds in the new context.

\section{Conclusion}
We have proposed a declarative, interventionally grounded layer for learned systems: mechanisms are expressed as constraints, interventions as surgical edits of those constraints, and equilibria as the engine that restores global consistency. Energy\textendash Structured Causal Models (E\textendash SCMs) thereby turn internal relations into editable, testable objects without rearchitecting the underlying computation.

Conceptually, E\textendash SCMs place symbolic content—causal relations, invariants, and counterfactual semantics—inside the same differentiable substrate that learns representations. Constraint energies carry the “rules”; optimization enforces them. Counterfactuals are instantiated via specific edits to energies rather than calls to an external logic module. This yields a structurally neuro–symbolic view without a brittle neural–symbolic interface, potentially preserving end-to-end training.

Explanations become operational. A candidate explanation of real or latent observations is a set of constraints whose equilibria reproduce observed regularities and support designated interventions. Criticism is enacted as surgery: rival hypotheses are implemented by editing specific energies, and the resulting equilibria adjudicate among them. Analytical tools make this process measurable: equilibrium solves reveal whether edits propagate along intended paths; latent Hessians quantify susceptibility and stability; and LAP diagnostics expose nonlocal influence and gauge slack.

E\textendash SCMs are best understood as an explanatory layer for deep systems, not a replacement. They surface and discipline the structural assumptions that standard training leaves implicit, aligning learned representations with modular, interventionally meaningful mechanisms. This supports an error-centric practice: propose constraints, intervene, measure, and retain only what survives severe tests.

Future work should expand the library of reusable mechanism templates, develop identifiability conditions under representational gauge, and embed E\textendash SCMs in continual-learning loops where representation, constraint sets, and intervention policies co-evolve. The aim is a single substrate that learns, explains, and improves by error correction—bringing compositional generalization and counterfactual competence within the same energy-based framework.

\subsubsection*{Outlook: Hypergraph generalization}
Energies in an E\textendash SCM already specify which variables must co-act: each term \(E_k(x_{S_k})\) names a scope \(S_k\) whose variables are constrained together. A natural next step is to record this scope explicitly as a directed hypergraph. In a hypergraph, a single edge can relate many variables at once; by giving each incidence a read or write role, an energy becomes a morphism \(R_k \Rightarrow W_k\), where \(R_k\) are the variables an energy reads to evaluate consistency and \(W_k\) are the variables it writes (sets at equilibrium). This representation exposes mechanism arity that a pairwise DAG flattens. For example, a shared-resource constraint that ties many consumers to a single capacity is naturally a single high-order scope rather than a bundle of pairwise links; likewise, a pure interaction such as parity (XOR) is recognizably synergistic only when its parents are treated as a joint unit. This view complements our reduction to SCMs: the DAG retains ancestry and identification tools, while the hypergraph records the scopes on which energies actually operate.

Framed this way, interventions become precise surgeries on write-ports. A hard action on \(X\) severs all writes into \(X\) and replaces them with a clamp; a soft action modifies only the energies that write to \(X\) while preserving other scopes. Because scopes are first-class objects, the meaning of an edit no longer depends on a particular coordinatization of latents: the same surgery applied to the same hyperedges yields the same counterfactual regardless of encoder reparametrization. This motivates a structural selection principle: prefer small, well-aligned scopes and admit higher-order scopes only when evidence requires them. Such scope discipline localizes dependence, counteracting fracture and entanglement in the learned latents while still accommodating genuine global constraints.

\subsubsection*{Analogy and CAP in the energy setting}
Analogies, as defined in Part~I, do not require hypergraphs. They are structure-preserving correspondences between domains that carry recipes to recipes and respect the map-then-compose \(\approx\) compose-then-map condition. In an E\textendash SCM, this means transporting a mechanism template so that the same surgeries on the source energies induce the same intervention responses after translation. The compositional autonomy principle enters as a stability requirement on learning: the realized maps associated with transported recipes should retain their effects on designated variables when additional modules are introduced or parameters are updated elsewhere, except where explicitly coupled. Stated this way, CAP becomes testable with abduction–intervention–prediction and does not depend on any particular coordinatization of the latents. A later hypergraph formulation can make the same commitments explicit by naming scopes and read/write roles, but the core notion of analogy—and CAP as its preservation criterion—already lives cleanly in the DAG+energy formalism developed here.

\newpage
\bibliographystyle{plainnat}
\bibliography{refs}

\newpage
\appendix
\newcommand{\PA}{\mathrm{PA}}

\section{Reduction Theorem for Energy-Structured SCMs}
\label{sec:appendix-reduction-esm}
\paragraph{Why a reduction?}
The reduction theorem shows that, under mild well-posedness and locality assumptions, an energy-structured causal model (E\textendash SCM) can be translated into an ordinary structural causal model (SCM) that behaves identically for purposes of abduction–intervention–prediction. Concretely, each local energy term induces a "best-response" map, and the global equilibrium coincides with the unique fixed point of the resulting structural equations; after any local surgery (hard or soft), the edited equilibrium and the edited fixed point still agree. This translation lets us inherit the mature SCM toolkit—counterfactual semantics, do-calculus/identification, and logical soundness/completeness—without re-proving everything inside the energy formalism. Intuitively, energies specify constraints rather than computations, but once equilibria are unique and modular, those constraints behave like structural equations. 

\paragraph{Notation.}
In this appendix we write $X=(X_1,\dots,X_n)$ for the endogenous variables (elsewhere denoted $Z$) and use $z_{\mathrm{PA}(i)}$ for parent sets; all results here are stated in the $X$–notation.

We formalize a subclass of energy-structured causal models (E\textendash SCMs) whose counterfactual semantics coincide with those of ordinary structural causal models (SCMs). Throughout, let $X=(X_1,\dots,X_n)$ be endogenous variables indexed by $[n]$, $U=(U_1,\dots,U_n)$ exogenous variables with joint law $P_U$, and let $G$ be a directed graph with parent sets $z_{\mathrm{PA}(i)}\subseteq [n]\setminus\{i\}$.

\paragraph{Model class.} An E\textendash SCM in this appendix is specified by a real-valued energy of the form
\begin{equation}
E(x;u) \;=\; \sum_{i=1}^n \phi_i\big(x_i,\, x_{z_{\mathrm{PA}(i)}},\, u_i\big),
\label{eq:separable-energy}
\end{equation}
where each local term $\phi_i:\mathcal{X}_i\times \mathcal{X}_{z_{\mathrm{PA}(i)}}\times\mathcal{U}_i\to\mathbb{R}$ is Borel-measurable. For each $u$, the realized world $x^\star(u)$ is the (selected) global minimizer of $E(\cdot;u)$.

\paragraph{Assumptions.} We separate pointwise (deterministic) and population (probabilistic) requirements. Assumptions (A1)--(A4) are used for the \emph{pointwise} reduction (no probability on $U$). Assumptions (A5)--(A6) are only needed when we add a probability law on $U$ to obtain distributional and counterfactual statements.
\begin{enumerate}[label=(A\arabic*)]
  \item \textbf{Locality.} Each $\phi_i$ depends only on $(x_i, x_{z_{\mathrm{PA}(i)}}, u_i)$ as in~\eqref{eq:separable-energy}.
  \item \textbf{Blockwise strict convexity.} For every $(x_{z_{\mathrm{PA}(i)}},u_i)$, the map $x_i\mapsto \phi_i(x_i, x_{z_{\mathrm{PA}(i)}}, u_i)$ is strictly convex on $\mathcal{X}_i$ with a unique minimizer.
  \item \textbf{Global strict convexity and coercivity.} For every $u$, the function $x\mapsto E(x;u)$ is strictly convex and coercive on $\mathcal{X}=\prod_i \mathcal{X}_i$ (hence admits a \emph{unique} global minimizer $x^\star(u)$).
  \item \textbf{Modular interventions (local surgery).} A hard intervention $\mathrm{do}(X_i{:=}x)$ is implemented by replacing $\phi_i$ with a term $\tilde{\phi}_i$ whose unique minimizer in its first argument is $x$ for all $(x_{z_{\mathrm{PA}(i)}},u_i)$, without altering any $\phi_j$ for $j\neq i$. Soft interventions edit $\phi_i$ to another local function $\tilde{\phi}_i$ while leaving $\{\phi_j\}_{j\neq i}$ unchanged.
  \item \textbf{Exogenous stability.} The distribution $P_U=\prod_{i=1}^n P_{U_i}$ is a product measure and remains invariant under interventions.
  \item \textbf{Measurability.} For each $i$, the argmin map
  \[
  f_i(x_{z_{\mathrm{PA}(i)}},u_i)\;:=\;\operatorname*{arg\,min}_{\tilde x_i\in\mathcal{X}_i}\; \phi_i(\,\tilde x_i,\, x_{z_{\mathrm{PA}(i)}},\, u_i\,)
  \]
  is Borel-measurable in $(x_{z_{\mathrm{PA}(i)}},u_i)$.
\end{enumerate}

\subsection*{Pointwise Reduction (No Probability)}
We begin with a deterministic semantics that fixes an exogenous realization $u$ and involves no probability on $U$.

\begin{proposition}[Pointwise reduction]\label{prop:pointwise-reduction}
Under (A1)--(A4), for each fixed $u$ the unique global minimizer $x^\star(u)$ of $E(\cdot;u)$ coincides with the unique solution of the fixed-point system
\[
X_i\;=\; f_i\big(X_{z_{\mathrm{PA}(i)}},u_i\big),\qquad i\in[n],
\]
where $f_i$ is the unique blockwise argmin of $\phi_i$. Moreover, for any finite set $I$ of intervened indices and any local hard/soft surgeries on $\{\phi_i\}_{i\in I}$, the edited minimizer equals the solution of the correspondingly edited fixed-point system (obtained by replacing $f_i$ by $\tilde{f}_i$ for $i\in I$ and leaving other $f_j$ unchanged).
\end{proposition}

\begin{proof}
Fix $u$.
\begin{enumerate}[label=(\roman*), leftmargin=*, itemsep=0.25em]
  \item By (A3), $E(\cdot;u)$ is strictly convex and coercive, hence has a unique minimizer $x^\star(u)$.
  \item For each $i$, fix the parents to their optimal values, $x_{z_{\mathrm{PA}(i)}} = x^\star_{z_{\mathrm{PA}(i)}}(u)$, and consider the one-variable slice
  \[
    x_i \;\longmapsto\; \phi_i\big(x_i,\, x^\star_{z_{\mathrm{PA}(i)}}(u),\, u_i\big).
  \]
  By (A2) this function is strictly convex in $x_i$ with a unique minimizer $f_i\big(x^\star_{z_{\mathrm{PA}(i)}}(u),u_i\big)$. If $x_i^\star(u) \neq f_i\big(x^\star_{z_{\mathrm{PA}(i)}}(u),u_i\big)$, replacing $x_i^\star(u)$ by that minimizer would strictly decrease $E(\cdot;u)$ while keeping all other coordinates fixed, contradicting minimality. Hence $x_i^\star(u)=f_i\big(x^\star_{z_{\mathrm{PA}(i)}}(u),u_i\big)$ for all $i$, so $x^\star(u)$ solves $X_i=f_i(X_{z_{\mathrm{PA}(i)}},u_i)$.
  \item Conversely, let $x$ solve $X_i=f_i(X_{z_{\mathrm{PA}(i)}},u_i)$ for all $i$. Then each coordinate is a blockwise minimizer; thus $x$ is a stationary point of the strictly convex function $E(\cdot;u)$, and strict convexity implies $x$ is the unique global minimizer. The intervention case is identical after replacing the edited local terms by their corresponding argmin maps.
\end{enumerate}
\end{proof}

\begin{theorem}[Pointwise counterfactual reduction]\label{thm:pointwise-cf}
Assume (A1)--(A4). Fix a context $u$ and any finite set $I\subseteq[n]$ of intervened indices (hard or soft local surgeries per (A4)). Let $x^{\star,(I)}(u)$ denote the unique minimizer of the edited energy and let $x^{\mathrm{scm},(I)}(u)$ denote the unique solution of the surgically edited fixed-point system. Then for every endogenous coordinate (and hence for any measurable function of $X$),
\[
X^{E,(I)}(u)\;=\;X^{S,(I)}(u),\qquad\text{in particular }\; Y^{E,(I)}(u)=Y^{S,(I)}(u).
\]
\end{theorem}

\begin{proof}
This is immediate from Proposition~\ref{prop:pointwise-reduction} applied to the edited local terms: both semantics solve the same edited fixed-point system for the same $u$.
\end{proof}

\begin{corollary}[Modal/set-valued counterfactuals without probability]\label{cor:modal-cf}
Assume (A1)--(A4). Given evidence $\mathcal E$ (no probability on $U$), let $\mathcal U(\mathcal E)$ be the set of contexts consistent with $\mathcal E$. Define the counterfactual images
\[
\mathcal C^{E}(I\mid \mathcal E):=\{\, X^{E,(I)}(u): u\in \mathcal U(\mathcal E)\,\},\qquad
\mathcal C^{S}(I\mid \mathcal E):=\{\, X^{S,(I)}(u): u\in \mathcal U(\mathcal E)\,\}.
\]
(We use superscript $E$ for the energy-structured model and $S$ for the induced SCM; thus $X^{E,(I)}(u)$ and $X^{S,(I)}(u)$ are the post-intervention solutions at context $u$ under each semantics.)

Then $\mathcal C^{E}(I\mid \mathcal E)=\mathcal C^{S}(I\mid \mathcal E)$. Consequently, necessity (``for all $u\in\mathcal U(\mathcal E)$, $Y^{(I)}=y$'') and possibility (``for some $u\in\mathcal U(\mathcal E)$, $Y^{(I)}=y$'') statements coincide in the two semantics. If, in addition, a deterministic selection rule on $\mathcal U(\mathcal E)$ is specified and used identically in both semantics, the selected counterfactuals coincide as well.
\end{corollary}

\begin{proof}
For each $u\in\mathcal U(\mathcal E)$, Theorem~\ref{thm:pointwise-cf} gives $X^{E,(I)}(u)=X^{S,(I)}(u)$. Taking images over $\mathcal U(\mathcal E)$ yields the equality of sets and the necessity/possibility equivalences. The selection-rule claim follows by applying the same deterministic map to identical sets.
\end{proof}

\begin{corollary}[Pushforward laws and counterfactuals]\label{cor:pushforward}
Assume (A1)--(A6). Let $g(u)=x^\star(u)$ and, for an intervention set $I$, let $g^{(I)}(u)$ denote the unique edited solution. Then: 
\begin{enumerate}[label=\alph*)]
  \item (Observational and interventional laws) The distributions of $X=g(U)$ and $X^{(I)}=g^{(I)}(U)$ equal, respectively, the observational and $\mathrm{do}$-distributions of the induced SCM. In particular, both are pushforwards of $P_U$ under the same measurable maps as in the SCM.
  \item (Counterfactuals) For any evidence event $\mathcal{E}$ measurable with respect to $(X,U)$, the abduction--action--prediction procedure in the E\textendash SCM (Bayesian abduction on $U$ under $P_U$, local surgery, forward solve) yields the same counterfactual distributions as in the induced SCM.
\end{enumerate}
\end{corollary}

\begin{proof}[Proof sketch]
By Proposition~\ref{prop:pointwise-reduction} and Theorem~\ref{thm:pointwise-cf}, for each $u$ the pointwise solutions agree in the original and induced models (before and after surgery). Measurability (A6) makes $g$ and $g^{(I)}$ Borel, so pushforward laws are well-defined and coincide in both models. Exogenous stability (A5) ensures abduction uses the same prior and that interventions do not alter $P_U$. The counterfactual claim follows by equality of the posterior over $U$ and equality of the forward maps.
\end{proof}

\subsection*{Population Layer (Optional)}
We now add a probability law on $U$ to obtain distributions over observables and counterfactuals. Assume (A5)--(A6) in addition to (A1)--(A4).

\begin{definition}[Induced SCM]
Under (A1)--(A6), define the \emph{induced SCM} $M_S$ by the structural equations
\begin{equation}
X_i \;:=\; f_i\big(X_{z_{\mathrm{PA}(i)}}, U_i\big), \qquad i\in[n],
\label{eq:induced-scm}
\end{equation}
with exogenous law $P_U$. Interventions on $M_S$ are defined by standard \emph{surgery}: for hard $\mathrm{do}(X_i{:=}x)$, replace the $i$th equation by $X_i:=x$; for soft edits of $\phi_i$ to $\tilde{\phi}_i$, replace $f_i$ by the corresponding argmin map $\tilde{f}_i$ and keep all other $f_j$ unchanged.
\end{definition}

\begin{theorem}[Reduction Theorem for convex, separable E\textendash SCMs]\label{thm:reduction}
Under (A1)--(A6), the E\textendash SCM defined by~\eqref{eq:separable-energy} is \emph{observationally, interventionally, and counterfactually equivalent} to its induced SCM~\eqref{eq:induced-scm} in the following precise senses:
\begin{enumerate}[label=\arabic*)]
  \item (Observational equivalence) For every $u$, the unique solution $x^{\mathrm{scm}}(u)$ of~\eqref{eq:induced-scm} equals the unique global minimizer $x^\star(u)$ of $E(\cdot;u)$. Hence the induced observational laws on $X$ coincide.
  \item (Interventional equivalence) For any finite set $I\subseteq[n]$ and any hard/soft local surgeries on $\{\phi_i\}\_{i\in I}$, the edited E\textendash SCM and the surgically edited SCM yield, for every $u$, the same unique solution; therefore all $\mathrm{do}$-distributions agree.
  \item (Counterfactual equivalence) For any evidential event $\mathcal{E}$ measurable w.r.t. $(X,U)$, the abduction--action--prediction procedure produces identical counterfactual distributions in both models.
\end{enumerate}
\end{theorem}

\begin{proof}
Immediate from Proposition~\ref{prop:pointwise-reduction} (pointwise equality of solutions before/after local surgery), Theorem~\ref{thm:pointwise-cf} (pointwise counterfactual equality at any context), and Corollary~\ref{cor:pushforward} (measurability and pushforward of $P_U$ yield equality of observational, interventional, and counterfactual distributions).
\end{proof}

\begin{remark}[On assumptions and alternatives]
Assumptions (A2)--(A3) are convenient sufficient conditions ensuring a unique solution, but they are not essential to the reduction. What the proofs actually require is the uniqueness of the fixed point of the blockwise argmin operator
\[
T(x)_i := \arg\min_{\tilde x_i \in \mathcal X_i}\, \phi_i\big(\tilde x_i,\, x_{z_{\mathrm{PA}(i)}},\, u_i\big),
\]
which coincides with the (selected) equilibrium. Global strict convexity and coercivity (A3) plus blockwise strict convexity (A2) guarantee this uniqueness in a simple, widely applicable way (e.g., strongly convex quadratics). However, other well-posedness packages also suffice:

(i) Contractive best response. If $T$ is a contraction under some norm (e.g., due to Lipschitz small cross-effects), Banach's fixed-point theorem gives a unique fixed point without global strict convexity.

(ii) Diagonal dominance / weak coupling. If each block is strictly convex (A2) and cross-dependencies are sufficiently weak (a diagonal-dominance condition), the blockwise argmin admits a unique fixed point.

Either alternative can replace (A2)--(A3) in our arguments. We use (A2)--(A3) because they keep the presentation short while covering many practical models; readers may substitute any condition that yields existence and uniqueness of the fixed point of $T$ (and hence of the equilibrium).
\end{remark}

\section{Hierarchy of Observable Sets (Technical details)}
\label{sec:appendix-hierarchy}

\paragraph{Minimal setting.}
Let $z$ range over an open set $\Omega\subset\mathbb{R}^d$ (a fixed coordinate chart). For each module $i$, $E_i\in C^2(\Omega)$. For a chosen probe family $H$, write $H(\{E_i\})$ for the numeric quantities returned by $H$ on $\Omega$. A transformation
\[
z'=\phi(z),\qquad E_i'(z') = a_i\,E_i\!\big(\phi^{-1}(z')\big)+b_i,
\]
with $\phi:\Omega\to\Omega$ a $C^2$ diffeomorphism, $a_i>0$, $b_i\in\mathbb{R}$, \emph{preserves $H$} if $H(\{E_i\})$ and $H(\{E_i'\})$ are equal as numeric functions on $\Omega$ (after pullback to the same coordinates when needed).

\paragraph{Probe families.}
We use the following heads (all numeric in the fixed chart):
\[
\begin{aligned}
H_E &: z\mapsto \{E_i(z)\}_{i=1}^n, \qquad
H_{\partial E} : z\mapsto \{\partial E_i/\partial z_j(z)\}_{i,j},\\
H_{\nabla E} &: z\mapsto \{\nabla E_i(z)\}_i, \quad
H_{\Delta E} : (z,z_0)\mapsto \{E_i(z)-E_i(z_0)\}_i,\\
H_{\mathrm{Hess}} &: z\mapsto \{\nabla^2 E_i(z)\}_i.
\end{aligned}
\]

\begin{proposition}[Identifiability under $H_E$]\label{prop:HE}
If a transformation preserves $H_E$ on a nonempty open set, then $a_i=1$ and $b_i=0$ for all $i$; moreover, unless $E_i\circ\phi^{-1}\equiv E_i$ on an open set, one must have $\phi=\mathrm{id}$.
\end{proposition}

\begin{proof}[Proof Sketch]
From $E_i'(z)=E_i(z)$ we get $a_i E_i(\phi^{-1}(z))+b_i=E_i(z)$. Evaluating at two points with different $E_i$ forces $a_i=1$, hence $b_i=0$. If also $E_i\circ\phi^{-1}\equiv E_i$ on an open set, generic unique–continuation arguments imply $\phi=\mathrm{id}$.
\end{proof}

\begin{proposition}[Identifiability under $H_{\partial E}$]\label{prop:dE}
If a transformation preserves $H_{\partial E}$ on an open set, then $b_i=0$ for all $i$. If magnitudes (not only directions) are preserved, then also $a_i=1$ for all $i$, and generically $\phi=\mathrm{id}$.
\end{proposition}

\begin{proof}[Proof Sketch]
Additive constants vanish: preserving $\partial E_i/\partial z_j$ forces $b_i=0$. Under the map,
$\partial_z E_i' = a_i\,(\partial_z E_i)\,D\phi^{-1}$. Equality of numeric derivatives on an open set yields $a_i=1$ and $D\phi^{-1}=I$, hence $\phi=\mathrm{id}$ generically.
\end{proof}

\begin{proposition}[What gradients, differences, and Hessians do not fix]\label{prop:scale-persist}
Preserving $H_{\nabla E}$, $H_{\Delta E}$, or $H_{\mathrm{Hess}}$ removes $b_i$ but leaves the per–module scales $a_i$ undetermined. For any $a_i>0$,
\[
\nabla(a_i E_i)=a_i \nabla E_i,\qquad
\nabla^2(a_i E_i)=a_i \nabla^2 E_i,\qquad
(a_i E_i)(z)-(a_i E_i)(z_0)=a_i\,[E_i(z)-E_i(z_0)].
\]
\end{proposition}

\begin{proof}
Direct computation.
\end{proof}

\begin{remark}[When scale becomes identifiable]\label{rem:scale}
Per–module scales are fixed by any external calibration of units, a normalization such as $\int e^{-E_i(z)}\,\mathrm{d}z=1$, or cross–module couplings that tie scales. Any of these breaks the $a_i$ ambiguity.
\end{remark}

\paragraph{Counterexample (scale non-identifiability).}
Let $E(z)=\tfrac12\|z\|^2$ on $\mathbb{R}^d$. For any $a>0$,
\[
\nabla(aE)=a\,z,\qquad \nabla^2(aE)=a\,I,\qquad (aE)(z)-(aE)(z_0)=a\,[E(z)-E(z_0)].
\]
Thus $H_{\nabla E}$, $H_{\Delta E}$, and $H_{\mathrm{Hess}}$ cannot recover $a$ without an external scale.

\paragraph{Coordinate–free note.}
If probes co–transform under $\phi$ (pushforward/pullback included in what is “observed”), then $\phi$ is, by design, part of the gauge and unidentifiable. The scale conclusions above are unchanged: differentials remove offsets, but scales require calibration as in Remark~\ref{rem:scale}.

\end{document}